\documentclass[runningheads]{llncs}

\usepackage[T1]{fontenc}
\usepackage{graphicx}

\usepackage{bbm}
\usepackage{amsmath,amssymb,mathtools}
\newcommand{\qedsymbol}{$\square$}
\AtEndEnvironment{proof}{\hfill\qedsymbol}

\usepackage{algorithm}
\usepackage{algpseudocode}
\algdef{SE}[DOWHILE]{Do}{doWhile}{\algorithmicdo}[1]{\algorithmicwhile\ #1}

\usepackage{dsfont}

\usepackage{xcolor}
\usepackage{enumitem}
\usepackage{url}

\usepackage{wrapfig}
\usepackage[caption=false]{subfig}

\usepackage{booktabs}
\usepackage{siunitx}
\usepackage{comment}

\usepackage{placeins}
\usepackage{soul,xcolor}

\usepackage[most]{tcolorbox}

\begin{document}
\title{Optimal Knock-Pick Planning for\\Tightly Packed Tabletop Blocks\\With Parallel Grippers}

\author{Hao Lu \and Rahul Shome}
\authorrunning{H. Lu et al.}
\titlerunning{Optimal Tabletop Manipulation of Tightly Packed Blocks}
\institute{School of Computing\\Australian National University}

\maketitle

\newcommand{\Z}{\mathbb{Z}}
\newcommand{\abs}[1]{\left|#1\right|}
\newcommand{\set}[1]{\left\{#1\right\}}
\newcommand{\cH}{\mathcal{H}}
\newcommand{\cB}{\mathcal{B}}
\newcommand{\cC}{\mathcal{C}}
\newcommand{\cS}{\mathcal{S}}
\newcommand{\cU}{\mathcal{U}}
\newcommand{\cF}{\mathcal{F}}
\newcommand{\cG}{\mathcal{G}}
\newcommand{\cV}{\mathcal{V}}
\newcommand{\cK}{\mathcal{K}}
\newcommand{\cM}{\mathcal{M}}

\newcommand{\EC}{\cC_{\mathrm{E}}}

\newcommand{\G}{G}
\newcommand{\E}{E}

\newcommand{\Clean}{\mathsf{Clean}}
\newcommand{\Core}{\mathsf{Core}}
\newcommand{\layer}{\lambda}

\newcommand{\Knock}{\mathsf{Knock}}
\newcommand{\Pick}{\mathsf{Pick}}

\newcommand{\OPT}{\mathsf{OPT}}
\newcommand{\ALG}{\mathsf{ALG}}

\newcommand{\Ray}{\mathsf{Ray}}
\newcommand{\Faces}{\mathsf{Faces}}
\newcommand{\Face}{\mathsf{face}}
\newcommand{\Cand}{\mathsf{Cand}}

\newcommand{\objects}{\mathcal{O}}
\newcommand{\object}{{o}}
\newcommand{\arrange}{\mathcal{P}}
\newcommand{\pose}{p}
\newcommand{\poseset}{P}
\newcommand{\actions}{\mathcal{A}}
\newcommand{\action}{{a}}
\newcommand{\sol}{\Pi}

\newcommand{\sideI}{h}
\newcommand{\vertI}{v}
\newcommand{\chor}{\mathcal{C}_{\sideI}}
\newcommand{\cvert}{\mathcal{C}_{\vertI}}
\newcommand{\ghor}{S_{\sideI}^\uparrow}
\newcommand{\gvert}{S_{\vertI}^\rightarrow}
\newcommand{\bchor}{{\mathcal{C}}_{\sideI}^{\rightarrow}}
\newcommand{\bcvert}{{\mathcal{C}}_{\vertI}^{\uparrow}}
\newcommand{\bghor}{S_{\sideI}^{\rightarrow\uparrow}}
\newcommand{\bgvert}{S_{\vertI}^{\uparrow\rightarrow}}

\sethlcolor{yellow!20}
\newcommand{\camera}[1]{\hl{#1}}

\newtcolorbox{crd}{
    colback=blue!5,
    colframe=blue!60!black,
    boxrule=0.8pt,
    arc=3pt,
    left=8pt, right=8pt, top=8pt, bottom=8pt,
    breakable
}

\newcommand{\cameraready}[1]{{#1}}

\newtheorem{assumption}{Assumption}

\begin{abstract}
Rearranging densely packed tabletop objects is challenging when parallel-gripper picks are infeasible without sufficient clearance around an object. This work studies the problem characteristics for practically motivated settings with uniformly sized blocks placed at planar tabletop grid locations. Since purely prehensile removal can become infeasible, a directional knock primitive is therefore introduced and the optimal knock-pick variant of the problem is formulated. The work proposes a series of abstractions wherein minimal constraining gadgets are covered to identify the necessary knocks. Utilizing maximum-weight perfect matching on a graphical abstraction yields efficient polynomial time computation of the optimal plan that minimizes the number of actions. Experiments are reported for increasing grid sizes in synthetic settings as well as in IsaacSim. The theoretical observations provide a promising stepping stone towards rigorously building efficient manipulation strategies that interleave prehensile and non-prehensile actions. 

\end{abstract}
\section{Introduction}

\begin{wrapfigure}[14]{r}{0.6\textwidth}
  \vspace{-0.45in}
  \centering
  \begin{minipage}[b]{0.71\linewidth} 
    \centering
    \subfloat{\includegraphics[trim={15cm 3cm 40cm 3.7cm},clip,width=\linewidth]{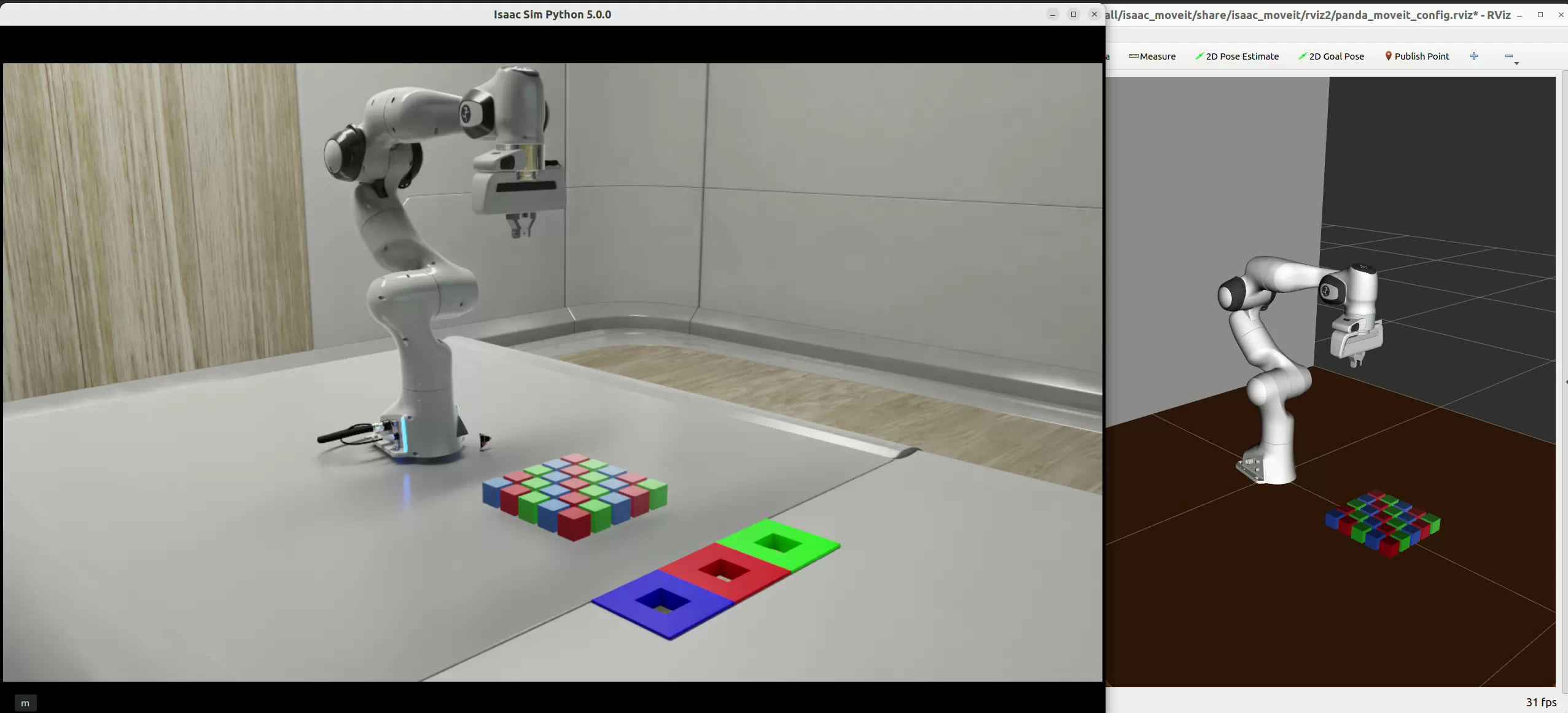}}
  \end{minipage}%
  \hfill
  \begin{minipage}[b]{0.28\linewidth}
    \centering
    \subfloat{\includegraphics[trim={28cm 12cm 51cm 20cm},clip,width=\linewidth]{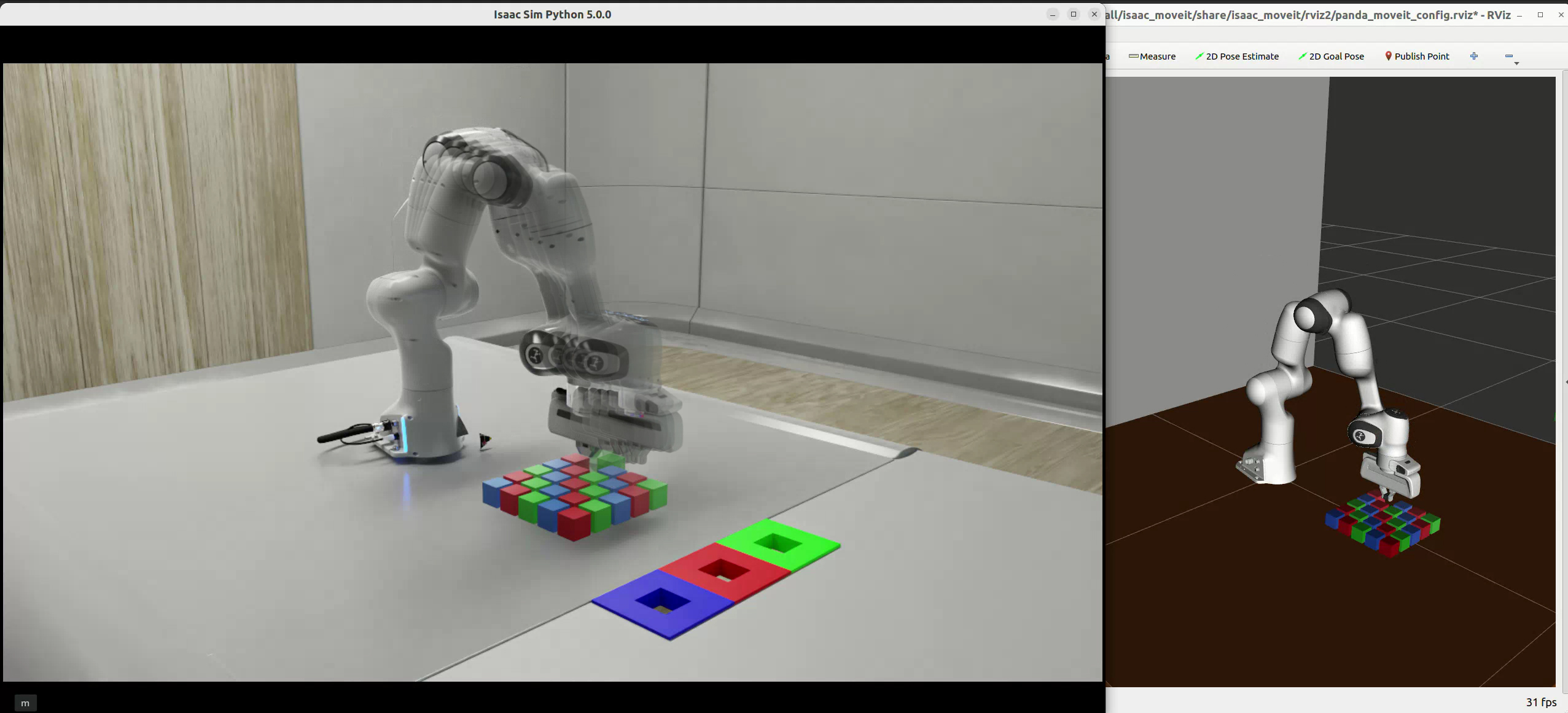}}\\
    \vspace{-10pt}
    \subfloat{\includegraphics[trim={28cm 12cm 51cm 20cm},clip,width=\linewidth]{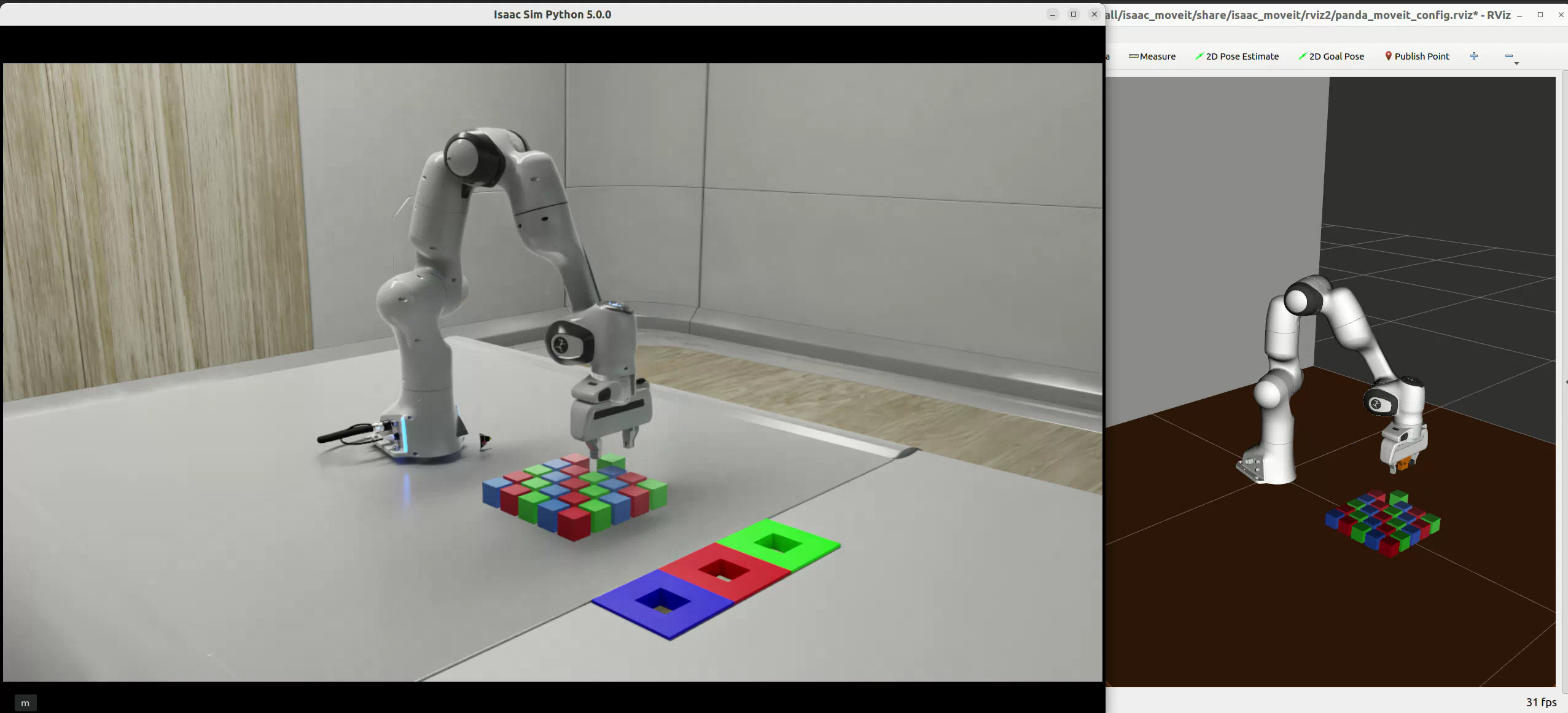}}\\
    \vspace{-10pt}
    \subfloat{\includegraphics[trim={28cm 12cm 51cm 20cm},clip,width=\linewidth]{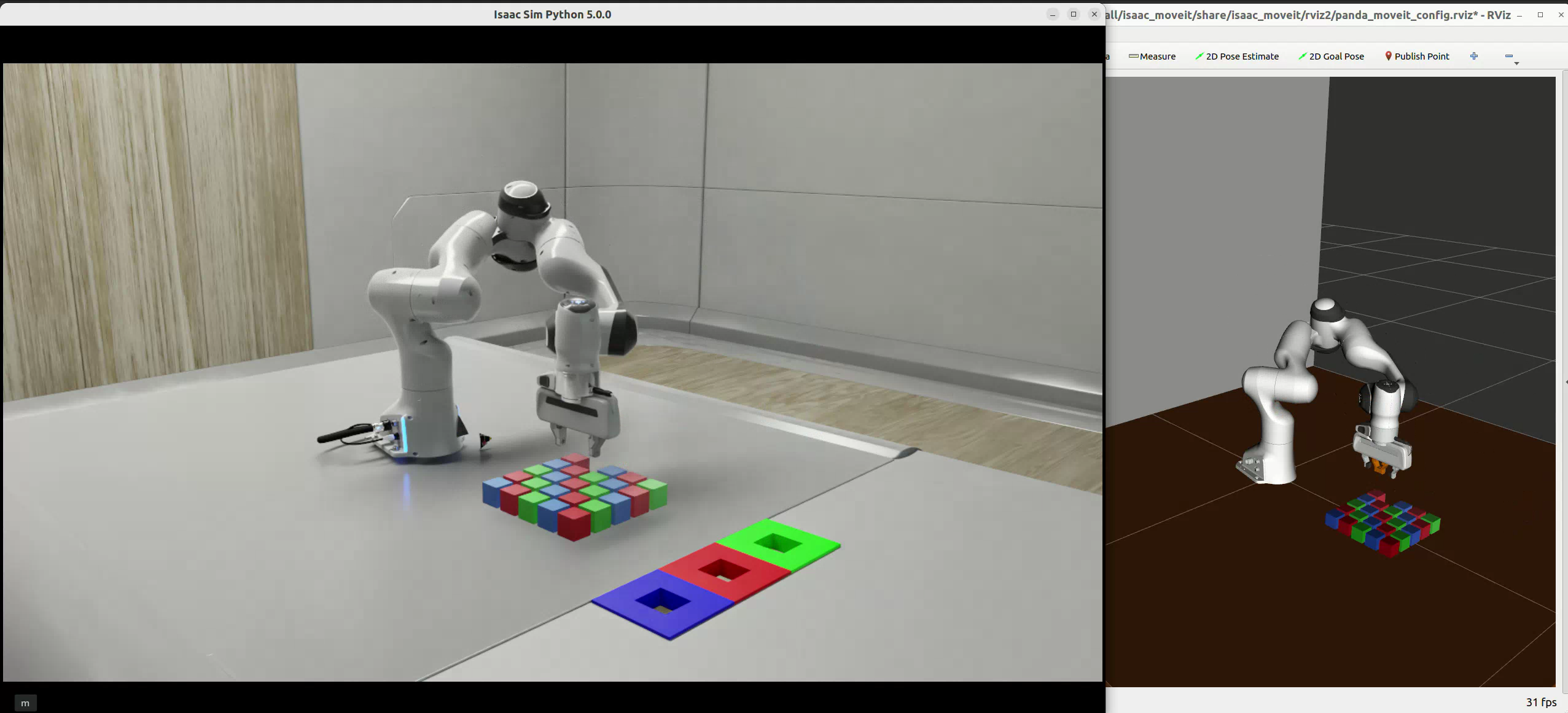}}
  \end{minipage}
  \vspace{-0.28in}
  \caption{A color sorting task with parallel grippers. \textit{(Right)} A \textit{knock} makes the subsequent \textit{picks} valid.}
  \label{fig:fig1}
  \vspace{-0.2in}
\end{wrapfigure}

Tightly packed objects are common in logistics and automation settings~\cite{shome2019towards,wang2021dense}. While object rearrangement~\cite{han2017complexity} is a well-studied problem with wide applications in automation, tightly packed settings like the sorting task in Fig~\ref{fig:fig1} expose a critical gap --- using parallel gripper picks is often infeasible when objects have no clearance. A combination of prehensile picks and some non-prehensile action is imperative. This work introduces a non-prehensile \textit{knock} action that displaces an object in a desired direction. It remains to be seen whether \textit{optimal knock--pick plans can be computed efficiently in tightly packed settings?}

Towards the broader aim, the current work focuses on a more specific setting. This work inspects the category of problems arising from using a parallel gripper to pick a set of tabletop cubical blocks that are tightly packed in a grid (Fig~\ref{fig:fig1}). Having a suction-based end-effector permits immediate feasibility with top-down picks and allows the problem to be solved and analyzed using existing techniques~\cite{han2017complexity}. In contrast, finger-based end-effectors, like parallel grippers, typically need clearance around objects to allow them to be picked. {Densely packed} arrangements often exhibit local geometric locking: surrounding objects prevent access to grasp points, and no object is immediately graspable with finger-based end-effectors. 
Such dense scenes can arise in warehouse bin clearing, tabletop packing/unpacking, and similar "blocks world" tasks, where objects lack the necessary clearance to admit pick feasibility with parallel grippers. Note that the occurrence of even a single square-like pattern of four blocks introduces infeasibility with parallel gripper picks.
It is of interest to solve and characterize the theoretical properties of sequential manipulation problems in such tightly packed scenarios.



In tightly packed tabletop settings,
\emph{non-prehensile} actions---pushing, toppling, or knocking---that create space by moving an object without a stable grasp.
While necessary in the current setting,
they come with tangible costs:
they are less reliable than grasping, more likely to cause collateral motion, and may require additional recovery steps for execution feasibility. 
In many systems, a non-prehensile action is best viewed as an \emph{expensive clearance-creation operation}, after which standard pick--place can proceed. Designing robust non-prehensile primitives is, by itself, a challenging problem and an active area of research~\cite{wen2022you,chi2023diffusion}. In the current work, the hand-crafted knock primitive is tested on IsaacSim and shows promising performance, with room for motivating improvements to robust execution by bridging the Sim-to-Real gap~\cite{hofer2021sim2real}. Despite this, the empirical results motivate the utility of optimal knock--pick plans for tightly packed blocks and thereby preserve the primary focus of the current work in characterizing optimal knock--pick planning complexity and solving strategies.


The focus of the current work on grid arrangements provides an opportunity to study the regime under simplified abstractions defined over the tabletop grid. Interestingly, the core result uncovers that \textit{computing the high-level optimal knock-pick plan can be done in \textbf{polynomial time}}. This follows from transforming the minimal knock action computation to a maximum weight perfect matching problem that yields a way to {break} every local pick constraint. Minimal constraining gadgets are identified and ensuring all objects are covered by such gadgets achieves our aim. Under, the reachable tabletop setting, the top-down primitives will be motion feasible. The current theoretical result for the optimal plan is essentially the best performance achievable given the available actions. Practical execution strategies of non-prehensile actions may introduce the need to retry or repeat actions or more advanced online adaptation and recovery strategies, which is beyond the current scope of the paper. The reported physics simulator results include rudimentary retry and resume strategies and provide promising motivating performance. 
This paper makes three contributions:
(i) a formal graph-based abstraction of tightly packed block manipulation with a parallel gripper using top-down picks and non-prehensile knock actions;
(ii) framing the problem of minimum-knock plan computation to a special case of optimal exact cover problem that can be solved in polynomial time using maximum-weight perfect matching;
and (iii) an execution routine that converts the combinatorial optimum into a knock--pick (KP)-feasible sequence that is evaluated in synthetic graphs and physics simulation in IsaacSim.

\section{Related Work}
\paragraph{Object Rearrangement.}
A large body of work addresses object rearrangement~\cite{Ota:2004mi}, with strong results for optimality and complexity~\cite{han2017complexity} in scenarios involving different number and constraints exhibited for prehensile object picks and placements. 
Notably, tightly packed block settings can be infeasible solely using prehensile parallel gripper grasps so novel solvers have to be devised beyond the scope of classical object rearrangement, including rigorous exploration of underlying combinatorial structures that can lead to effective solvers. 

\paragraph{Cluttered Manipulation}
Cluttered settings~\cite{quintero2023optimal} have also been optimized~\cite{toussaint2015logic} in task and motion planning frameworks, particularly with parallel grippers. However, the focus here is problems where low but existing clearance and objects having different heights permitted feasibility with parallel grasp picks. Other work~\cite{havur2014geometric} has focused on geometric considerations of the problem. Task and motion planning~\cite{dantam2018incremental} paradigms provide generalizability and are complementary to effective combinatorial problem abstractions.   

\paragraph{Non-prehensile Manipulation}
Non-prehensile object interaction~\cite{song2020multi,huang2019large,song2019object} has attracted sustained interest since the seminal advances in push-grasping~\cite{Ben-Shahar:1998pi,dogar2011framework}. 
Among early works,~\cite{huang2000mechanics} analyzes the mechanics of planar sliders under \emph{tapping}, while~\cite{lynch1999toppling} formalizes \emph{toppling} as a manipulation primitive and derives the associated mechanical conditions and planning considerations.
Categories of manipulation problems require a combination of prehensile and non-prehensile actions. The current problem lies in such a category. Certain topological results have led to enhanced insights into the nature of push-grasp problems~\cite{vieira2022persistent}. The current work foregoes generalization for aiming to understand theoretical optimality bounds for such problem categories. Towards this end, the current study focuses on a less general, yet practical setting of tight block arrangements. Here, non-prehensile actions are necessary and tight theoretical bounds of polynomial solvability and optimality of task planning can be devised. 

\paragraph{Packing Assemblies} Beyond clutter, tightly packed objects~\cite{wang2021dense} are recognized to be a particularly challenging category of object manipulation problems where the clearance is low to non-existent. While pipelines~\cite{shome2019towards} designed for packing have been designed to compose such tightly packed arrangements, taking apart such an arrangement forms the focus of the current work. Interestingly, packing tasks have been often addressed with suction-based end-effectors as they typically do not depend on clearance around the object like parallel grippers. The current work attempts to look into this somewhat understudied application of parallel grippers to packing tasks. There is a rich history of related planning problems with computational geometric motion space modeling of assembly planning~\cite{halperin1998general} towards designing execution pipelines~\cite{suarez2016framework,tian2024asap}. The current work is a specific investigation of theoretical optimality for tightly packed uniform geometry blocks that need both prehensile and non-prehensile actions with generalization to other geometries and integrations of robust primitive design and execution pipelines being a motivating future direction.   

\paragraph{Learning Primitives} Non-prehensile grasps can exacerbate the sim-to-real gap~\cite{hofer2021sim2real} as complex object interactions introduces considerations of contact, friction, physics, etc. Learning from demonstration~\cite{wen2022you} can play a key role in robust non-prehensile primitives. Recent advances in imitation learning methods based on using generative models~\cite{chi2023diffusion} have shown promise. Effective primitive design, though not the focus of the current paper, is essential for practical robust deployment. The parameterized knock action has been hand-crafted and advanced learned controllers can be foreseeably plugged into robust online frameworks.

\section{Problem Formulation}
\label{sec:problem}


We consider a robot rearrangement manipulator motivated by tightly packed tabletop scenes. A set of identical cubical objects is placed on a planar tabletop and arranged on a discrete axis-aligned grid. A parallel-jaw gripper can grasp an object only when sufficient lateral clearance exists along at least one grasping axis. Otherwise, non-prehensile actions such as \emph{knocking} can create clearance.

The problem is set up with a manipulator with parallel gripper end effectors and a set of $k$ objects $\object$ at an initial arrangement $\arrange_{\mathrm{init}} = ( \pose_i, \pose_i \in SE(3) )$ such that $\pose_i$ is an initial pose of $\object_i$. The manipulator can interact the objects using a set of manipulation actions $\actions = \{ \action_j \}$ that correspond to a motion of the manipulator and alter the pose of the objects. A goal arrangement is defined in terms of goal sets per object, $\arrange_{\mathrm{final}} = ( P_i, P_i \subset SE(3) )$, a feasible solution to the manipulation problem is a plan $\sol = (\action_1, \action_2, \cdots \action_\ell)$ that moves the objects from the initial arrangement to a desired goal arrangement.  

\begin{assumption}[Initial tight grid arrangement of blocks]
The objects are assumed to be all cubical blocks of identical \textit{block} geometry.
Each pose in $\arrange_{\mathrm{init}}$ lies on a tabletop grid location where the object (in isolation) can be top-down grasped using parallel grippers. The dimension of the grid is equal to the dimension of a block, i.e., there is no clearance if adjacent locations are occupied. 
\label{ass:initgrid}
\end{assumption}
\vspace{-0.5em}
The no-clearance scenario is a key motivation for devising strategies with parallel grippers, which need clearance for grasp feasibility.
\begin{assumption}[Reachable unconstrained target arrangement]
The feasibility of placing an object at a target location does not depend on the initial block arrangement, i.e., there will exist feasible reachable candidates from $\arrange_{\mathrm{final}}$. 
\label{ass:targetarr}
\end{assumption}
\vspace{-0.4em}
The target arrangement assumption is easily motivated by categories of problems where the target locations are reachable but non-overlapping with $\arrange_{\mathrm{init}}$.
\vspace{-0.4em}
\begin{assumption}[Local non-prehensile primitive]
A non-prehensile primitive (knock) is assumed to be available as an action in $\actions$ that can act locally on a block, alter the pose of the block to make it feasible to pick, and only affect the pick feasibility of the current block and a local neighborhood.
\label{ass:localknock}
\end{assumption}
\vspace{-0.4em}
A non-prehensile action is needed in $\actions$ but its locality restriction allows a systematic inspection of the underlying combinatorial structure of the problem. Intuitively, the assumption distinguishes knock actions that move a block deliberately to an adjacent region from alternative strategies that might want to move larger groups of objects simultaneously, with the current work focusing on the former. 
The remaining formulation of the problem can proceed in the context of these assumptions.
Set the tabletop grid arrangement by fixing integers $m,n\ge 1$ and define the (axis-aligned) \emph{workspace grid} on the tabletop
\[
\cH \;\triangleq\; \{(i,j)\in\Z^2:\; 0\le i \le m-1,\; 0\le j \le n-1\}.
\]

\begin{definition}[Object Location Graph]
    
An \emph{grid location graph} is defined in terms of a subset $\cB\subseteq \cH$, where each $v\in\cB$ represents one cubical object placed at grid location $v$. 
Given $\cB\subseteq \cH$, define an undirected graph
\[
\G(\cB)\;\triangleq\;(\cV,\E),\qquad
\cV=\cB,\qquad
\{u,v\}\in\E \iff \|u-v\|_1=1.
\]
\end{definition}

\begin{wrapfigure}[3]{r}{0.2\linewidth}
    \centering
    \vspace{-35pt}
    \includegraphics[width=0.95\linewidth]{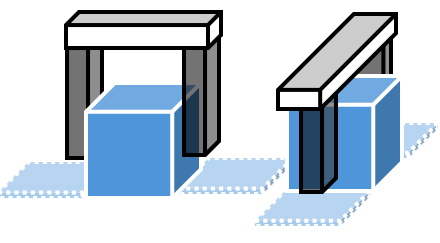}
    \vspace{-10pt}
    \caption{Parallel gripper $\Pick$.}
    \label{fig:pgpick}
    \vspace{-10pt}
\end{wrapfigure}

\subsection{Manipulation primitives}
It is necessary to model two manipulation primitives:
\emph{pick} (prehensile grasp with a parallel gripper as shown in Fig~\ref{fig:pgpick}) and \emph{knock} (non-prehensile removal along a grid-aligned direction as shown in Fig~\ref{fig:pgknock}).

\begin{wrapfigure}[13]{r}{0.2\linewidth}
    \centering
    \vspace{-20pt}
    \includegraphics[width=0.95\linewidth]{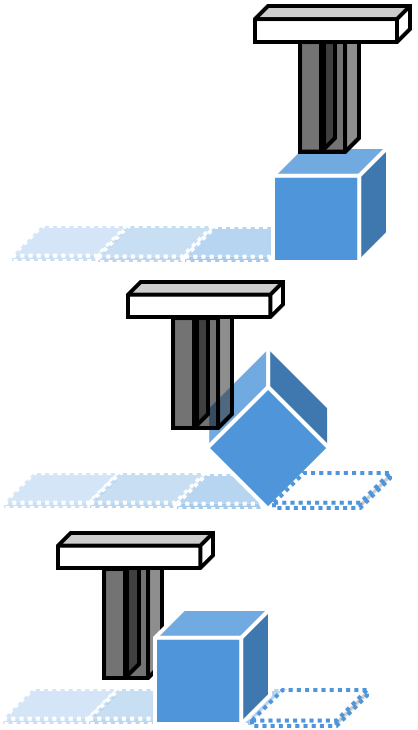}
    \vspace{-10pt}
    \caption{$\Knock$ action (to the left).}
    \label{fig:pgknock}
    \vspace{-10pt}
\end{wrapfigure}
\noindent\textbf{Pick feasibility. }
An object at a grid location $v\in\cV$ is \emph{pickable} if two adjacent locations are unoccupied in the same row or the same column. This leads to two conditions for feasibility in terms of vertex degree: (P1) $\deg_{\G}(v)\le 1$, or (P2) $\deg_{\G}(v)=2$ and its two neighbors are collinear.
We denote this predicate by $\Pick(\G,v)$.

\noindent\textbf{Knock feasibility. }
Let $\mathcal{D}\triangleq\{(\pm1,0),(0,\pm1)\}.$
For $v\in \cV $ and $d\in\mathcal{D}$, define the \emph{workspace ray}
$
\Ray(v,d)\;\triangleq\;\{v+td\in\cH:\; t\in\Z,\; t\ge 1\}.
$
A knock at $v$ in direction $d$ is valid if
\begin{equation}\label{eq:rayclear}
\Ray(v,d)\cap \cV = \emptyset.
\end{equation}
We say $v$ is \emph{knockable} if $\exists d\in\mathcal{D}$ satisfying~\eqref{eq:rayclear}, and denote this by $\Knock(\G,v)$.

A \emph{knock--pick} KP-plan for $\cB$ is a finite sequence
of actions which can be $a_t\in\{\Pick,\Knock\}$ for an object at vertex $v_t$ for each discrete step $t$.

The action at time $t$ is feasible if the corresponding predicate $\Pick(\G_t,v_t)$ 
or
$\Knock(\G_t,v_t)$ holds for all the actions sequentially across time steps.
The execution is \emph{complete} if it removes all vertices from $\G(\cB)$.

We measure cost by the number of actions in the plan.
This is equivalent to minimizing the number of knocks, since every object must be removed exactly once and picks are unavoidable whenever feasible.

\begin{definition}[Optimal KP Planning]
    Given an initial tabletop tightly packed grid arrangement of blocks represented by $\cB$, target goal arrangement $\arrange_{\mathrm{final}}$, available actions $\actions = \{\Knock, \Pick\}$, an optimal KP plan $\Pi = (\action_i)$ is a sequence of feasible actions that enact the rearrangement for all blocks such that the number of actions in $\Pi$ is minimized. 
\end{definition}

Since each object is eventually removed exactly once, minimizing the total number of actions is equivalent to minimizing the number of knocks, up to the additive constant $|\cV|$.

\FloatBarrier 

\section{Methodology}
\label{sec:method}
In this section the graphical abstractions to represent the problem will be proposed and solution strategies built on top of the combinatorial structures these abstractions will express.
Our approach proceeds in four conceptual steps:
(i) remove all immediately pickable objects by exhaustive picking (\emph{cleanup}),
(ii) characterize the remaining non-pickable structure via three minimal gadgets and a face system,
(iii) compute an \emph{optimal exact face cover} by reducing to a maximum-weight perfect matching problem, and
(iv) generate a \emph{KP-feasible} execution by iterating ray-feasible knocks and cleanup picks.

\begin{figure}[t]
    \centering
    \includegraphics[width=1\linewidth]{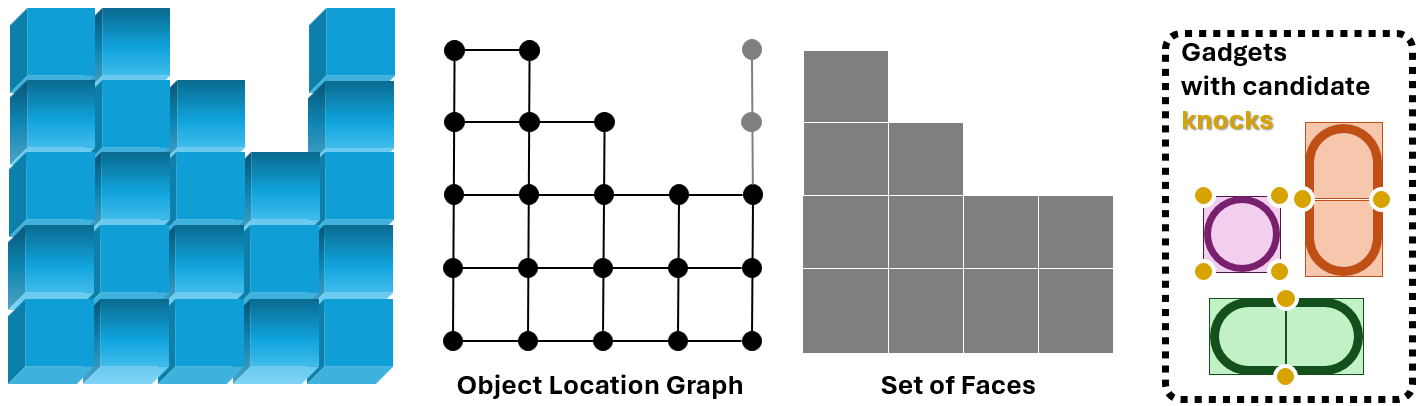}
    \vspace{-15pt}
    \caption{\textit{(Left)} Tabletop grid arrangement of blocks. \textit{(Second)} Object location graph, which greyed out ones indicating blocks that can be immediately picked (within $\Clean$). \textit{(Third)} A face is a minimal constraining set of four objects arranged in a $2\times 2$ square in the object location graph. The object location graph corresponds to a set of faces. \textit{(Right)} Gadgets are defined in terms of faces alongside knock candidate objects, knocking one of which can help pick the remaining objects in the gadget. }
    \label{fig:graphstructs}
    \vspace{-10pt}
\end{figure}
\begin{definition}[Cleaned Graph]
Given a graph $\G=(\cV,\E)$, define $\Clean(\G)$ as the graph obtained by repeatedly deleting any pickable vertex until none remains.
We call $\G$ \emph{cleaned} if $\G=\Clean(\G)$.
\end{definition}
In what follows, we write $\G_0\triangleq \G(\cB)$ for the initial grid graph and
$
\G^\circ \;\triangleq\; \Clean(\G_0)
$
for its cleaned graph.

\subsection{Faces and Gadgets}

We define three vertex-induced gadgets that represent minimal non-pickable patterns (illustrated in Fig~\ref{fig:graphstructs}).
Each gadget corresponds to either one or two \emph{unit faces}. 
A \emph{unit face} (or simply \emph{face}) is a $2\times 2$ fully occupied square.

\begin{definition}[Faces]\label{def:faces}
Given $\G(\cB)$, define the set of \emph{unit faces}
\[
\Faces(\G)\;\triangleq\;\Bigl\{(i,j)\in\Z^2:\;
\{(i,j),(i+1,j),(i,j+1),(i+1,j+1)\}\subseteq \cV(\G)\Bigr\}.
\]
Each $(i,j)\in\Faces(\G)$ identifies the corresponding $2\times 2$ vertex set. We denote this face by $\Face(i,j)$.
\end{definition}


\begin{definition}[Gadgets]\label{def:gadgets}
Let $(i,j)\in\Z^2$.

\noindent\textbf{Square gadget ($\mathbf{2\times 2}$).}
$
Q(i,j)\triangleq \{(i,j),(i+1,j),(i,j+1),(i+1,j+1)\}.
$
This corresponds to exactly one face $\Face(i,j)$.

\noindent\textbf{Vertical gadget ($\mathbf{3\times 2}$).}
$
R_v(i,j)\triangleq \{(i+r,j+c): r\in\{0,1,2\},\, c\in\{0,1\}\}.
$
This contains exactly two faces: $\Face(i,j)$ and $\Face(i+1,j)$.

\noindent\textbf{Horizontal gadget ($\mathbf{2\times 3}$).}
$
R_h(i,j)\triangleq \{(i+r,j+c): r\in\{0,1\},\, c\in\{0,1,2\}\}.
$
This contains exactly two faces: $\Face(i,j)$ and $\Face(i,j+1)$.
\end{definition}


To destroy \emph{two} faces with \emph{one} knock in a $2\times3$ or $3\times2$ gadget, the knocked vertex must lie on the shared boundary.
This yields the following candidate sets.

\begin{definition}[Knock-candidate vertices of a gadget]\label{def:cand}
Let $S$ be a gadget occurrence present in the current graph, and let $(i_0,j_0)$ be its anchor
(the lexicographically minimum coordinate in $S$).
Define $\Cand(S)\subseteq S$ by gadget type:
\begin{enumerate}[label=(\roman*)]
\item If $S=Q(i_0,j_0)$ is a $2\times2$ gadget, then $\Cand(S)=S$. 
\item If $S=R_h(i_0,j_0)$ is a $2\times3$ gadget, then
\[
\Cand(S)=\{(i_0,j_0+1),(i_0+1,j_0+1)\},
\]
i.e., the two vertices in the \emph{middle column} (shared boundary of the faces).
\item If $S=R_v(i_0,j_0)$ is a $3\times2$ gadget, then
\[
\Cand(S)=\{(i_0+1,j_0),(i_0+1,j_0+1)\},
\]
i.e., the two vertices in the \emph{middle row} (shared boundary of the faces).
\end{enumerate}
\end{definition}


\begin{lemma}[Gadgets are non-pickable]\label{lem:gadget-nonpick}
Let $X\in\{Q(i,j),R_v(i,j),R_h(i,j)\}$ be fully occupied and let $\G_X$ be the induced subgraph on $X$.
Then $\G_X$ has no pickable vertex: $\Pick(\G_X,v)$ is false for all $v\in X$.
\end{lemma}

\begin{proof}
In a $2\times2$ square, each vertex has degree $2$ and its two neighbors are non-collinear, so (P2) fails.
In a $3\times2$ (resp.\ $2\times3$) rectangle, corners again have degree $2$ with non-collinear neighbors, while the two interior boundary vertices have degree $3$.
Thus no vertex satisfies (P1) or (P2).
\end{proof}

\begin{lemma}[One knock suffices for each gadget]\label{lem:gadget-oneknock}
Let $X$ be fully occupied, where $X\in\{Q(i,j),R_v(i,j),R_h(i,j)\}$.
There exists a single vertex deletion (a knock) after which exhaustive picking removes all remaining vertices in $X$.
\end{lemma}

\begin{proof}
For $Q(i,j)$, knocking any one vertex leaves a 3-vertex tree, which is removable by leaf picking.
For $R_v(i,j)$ and $R_h(i,j)$, knocking an appropriate "hinge" vertex breaks both faces at once, leaving an acyclic structure with leaves.
Exhaustive picking then removes all remaining vertices.
\end{proof}

\subsection{Exact Covers with Gadgets}
Let $\G^\circ\triangleq \Clean(\G_0)$ denote the cleaned graph.
Define the universe as the faces in $\G^\circ$:
$
\cU \;\triangleq\; \Faces(\G^\circ).
$
For each gadget in $\G^\circ$, define a set of faces it covers:
\begin{itemize}[leftmargin=*]
\item If $Q(i,j)\subseteq V(\G^\circ)$ then include $S=\{\Face(i,j)\}$.
\item If $R_v(i,j)\subseteq V(\G^\circ)$ then include $S=\{\Face(i,j),\Face(i+1,j)\}$.
\item If $R_h(i,j)\subseteq V(\G^\circ)$ then include $S=\{\Face(i,j),\Face(i,j+1)\}$.
\end{itemize}
Let $\cS$ be the family of all such sets. Each $S\in\cS$ has size $1$ or $2$.

{

The current problem then needs to describe a set of such gadgets that covers all faces. Minimizing the cardinality of this cover correspondingly minimizes the number of actions in the original problem. 
}


\begin{definition}[Optimal exact face cover]\label{def:exact-cover}
An \emph{exact cover} is a subfamily $\cC\subseteq \cS$ such that
(i) $\bigcup_{S\in\cC} S=\cU$ and (ii) the sets in $\cC$ are pairwise disjoint:
$S\cap S'=\emptyset$ for all distinct $S,S'\in\cC$.
We seek an optimal exact cover
\[
\cC^\star \;\in\; \arg\min_{\cC\subseteq \cS}\; |\cC|
\quad \text{s.t. $\cC$ is an exact cover of $\cU$.}
\]
\end{definition}

\paragraph{Existence.}
An exact cover always exists for system $(\cU, \cS)$: for every face $f\in\cU$, the square gadget supplies a singleton set $\{f\}\in\cS$, so selecting all singletons yields an exact cover.


\begin{lemma}[Any face requires a knock]\label{lem:face-needs-knock}
Let $\G$ be cleaned (i.e., $\G=\Clean(\G)$) and let $\Face(i,j)\in\Faces(\G)$.
Then in any KP execution, the first removal of a vertex from $Q(i,j)$ must be a knock.
\end{lemma}

\begin{proof}
By Lemma~\ref{lem:gadget-nonpick}, no vertex of $Q(i,j)$ is pickable while the face is intact.
Since $\G$ is cleaned, no vertex is pickable at this stage, so the first removal from the face cannot be a pick and must be a knock.
\end{proof}

\begin{lemma}[A knock destroys at most two faces]\label{lem:knock-two-faces}
In a grid-induced graph, deleting one vertex destroys at most two faces of $\Faces(\cdot)$.
Moreover, it destroys two faces iff the deleted vertex belongs to two adjacent faces, i.e., it is a shared corner of a fully occupied $2\times3$ or $3\times2$ gadget.
\end{lemma}

\begin{proof}
A face is a $2\times2$ set of vertices, and a vertex participates in a face only as a corner.
If a vertex is knockable under~\eqref{eq:rayclear}, it has at least one missing incident neighbor, hence degree at most $3$.
In the grid, a degree-$\le 3$ vertex can be a corner of at most two fully occupied $2\times2$ faces (those using the two perpendicular present edges adjacent to the missing direction).
Thus deleting one vertex destroys at most two faces.
Two faces are destroyed exactly when the vertex is shared by two faces, which occurs precisely in the fully occupied $2\times3$ or $3\times2$ pattern.
\end{proof}



\subsection{Optimal Exact Cover via Maximum Weight Perfect Matching}

\begin{figure}[t]
    \centering
    \includegraphics[width=0.8\linewidth]{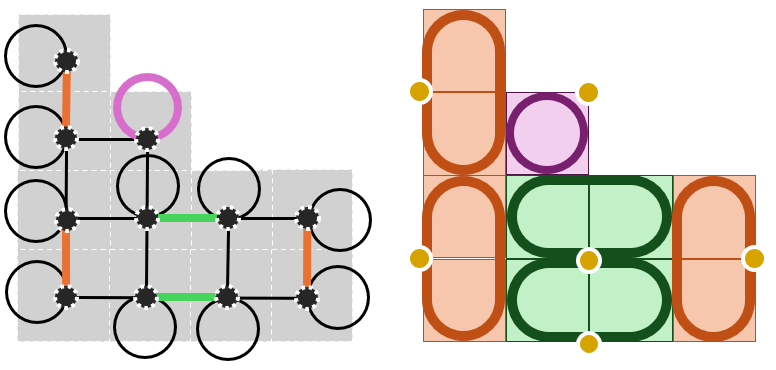}\\
    \includegraphics[width=0.95\linewidth]{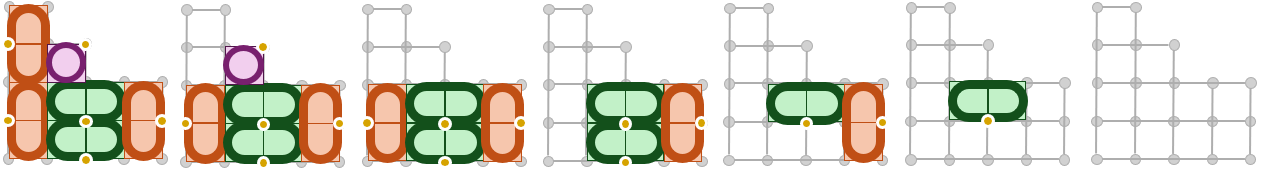}
    \vspace{-10pt}
    \caption{
    \textbf{(Top)} Face-graph abstraction (left) and the resulting optimal gadget cover (right).
    \textit{(Top-left)} Each vertex represents a unit face in the cleaned instance. A self-loop at a vertex represents a square ($2\times2$) gadget covering that single face; an edge between two vertices represents a horizontal ($2\times3$) or vertical ($3\times2$) gadget covering the corresponding adjacent face pair. Highlighted edges (including loops) show the maximum-weight perfect matching solution, which induces an optimal exact face cover.
    \textit{(Top-right)} The corresponding optimal gadget cover on the object graph (each vertex represents an object); highlighted object vertices indicate knock-candidate locations with at least one feasible knock direction in a KP-feasible execution.
    \textbf{(Bottom)} An illustration of a KP feasible plan corresponding to a sequence of gadgets knocked and cleaned till all objects are handled.
    }
    \label{fig:defgadget}
    \vspace{-10pt}
\end{figure}
Set cover itself is an NP-hard problem. However, the set system $(\cU,\cS)$ is a special case of set cover in which every set has size $1$ or $2$ and every element admits a singleton set.
This special structure lets us solve the exact set cover \emph{optimally} in polynomial time via maximum-weight perfect matching (see Fig~\ref{fig:defgadget}).

\begin{definition}[Face-adjacency graph with loops]\label{def:face-adj-graph}
Define the \emph{face-adjacency graph} $\cG_F$ as follows.
Its vertex set is $\cU=\Faces(\G^\circ)$.
For two distinct faces $f,g\in\cU$, add an undirected edge $\{f,g\}$ iff $f$ and $g$ are adjacent faces
(i.e., they share a grid edge) and the corresponding $2\times3$ or $3\times2$ gadget is present in $\G^\circ$.
Additionally, add a self-loop $(f,f)$ at every $f\in\cU$ corresponds to $2\times 2$ gadgets.
Assign weights
\[
w(f,f)=0
\quad\text{and}\quad
w(f,g)=1 \;\;\text{for all distinct adjacent $f,g$}.
\]
\end{definition}

\paragraph{Perfect matchings with loops.}
A \emph{perfect matching} of $\cG_F$ is a set of edges $\cM\subseteq E(\cG_F)$ such that every vertex $f\in\cU$ is incident to \emph{exactly one} edge in $\cM$, where a loop $(f,f)$ counts as incident to $f$.


\begin{lemma}[Exact covers $\leftrightarrow$ perfect matching]\label{lem:exactcover-matching}
There is a bijection between exact covers $\cC\subseteq\cS$ of $\cU$ and perfect matching $\cM$ of $\cG_F$. Under this bijection, each chosen singleton set $\{f\}$ corresponds to the loop $(f,f)$, and each chosen 2-face set $\{f,g\}$ corresponds to the edge $\{f,g\}$.
\end{lemma}
\begin{proof}
Given an exact cover $\cC$, every $f\in\cU$ belongs to exactly one chosen set.
Map that set to the corresponding loop (if singleton) or adjacency edge (if a pair). The disjointness of $\cC$ implies no two mapped edges share a vertex, and coverage implies each vertex is incident to exactly one mapped edge. Hence we obtain a perfect matching. Conversely, given a perfect matching, map each loop/edge to its singleton/pair set. The matching property yields the exactness of the cover and vertex coverage yields $\bigcup \cC=\cU$.
\end{proof}



\begin{lemma}[Optimal exact cover via maximum weighted perfect matching] \label{lem:optexactcover-matching}
The minimum exact cover $\cC^*\subseteq\cS$ of $\cU$ corresponds to a maximum weighted perfect matching $\cM^*$ of $\cG_F$. 
\end{lemma}
\begin{proof}
For any perfect matching $\cM$, let $p(\cM)$ be the number of non-loop edges in $\cM$.
Because $\cM$ covers all vertices exactly once, the number of loops is $|\cU|-2p(\cM)$, hence its weight is
\[
w(\cM)=1\cdot p(\cM) + 0\cdot(|\cU|-2p(\cM)) = p(\cM).
\]
Therefore maximizing $w(\cM)$ is equivalent to maximizing $p(\cM)$, i.e., using as many 2-face gadgets as possible.
Since any exact cover $\cC$ satisfies $|\cC|=p(\cM)+(|\cU|-2p(\cM)) = |\cU|-p(\cM)$ under the bijection,
a maximum-weight perfect matching $\cM^*$ yields a minimum-cardinality exact cover $\cC^\star$.
\end{proof}

\noindent\textbf{Polynomial-time solvability: }
Maximum-weight matching in general graphs is solvable in polynomial time by Edmonds' matching framework~\cite{galil1986efficient}. Here, $\cG_F$ always admits a perfect matching because every vertex has a loop available.



\subsection{Optimal knock minimization and KP-feasibility}\label{sec:opt-knocks}
We now connect the combinatorial optimum to the KP objective (minimum number of knocks) and show that the resulting optimal gadget set is always executable under the ray constraint.

\begin{figure}[t]
    \centering
    \includegraphics[width=0.9\linewidth]{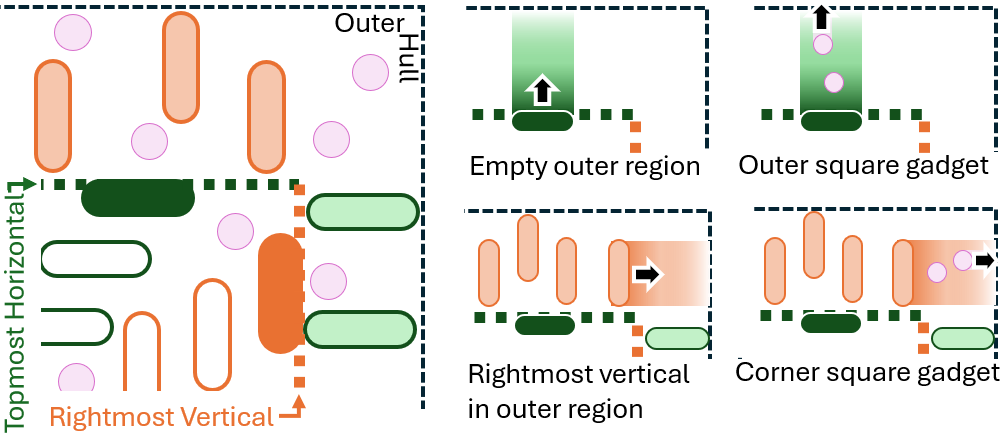}
    \vspace{-8pt}
\caption{
Illustration of Lemma~\ref{lem:sigma-outer-feasible} for the case $\sigma=(-1,1)$, corresponding to the top-right corner.
\textit{Left:} the horizontal and vertical $\sigma$-outer boundaries are determined by the extremal horizontal gadget $S_h^\sigma$ (dark green) and extremal vertical gadget $S_v^\sigma$ (dark orange). Horizontal gadgets are shown in green, vertical gadgets in orange, and square gadgets in pink.
\textit{Right:} representative configurations: an unobstructed boundary gadget (top-left); a boundary gadget whose first blocker and terminal gadget are both square (top-right); a boundary gadget whose first blocker and terminal gadget are both vertical (bottom-left); and a boundary gadget whose first blocker is vertical while its terminal gadget is square (bottom-right). In each subfigure, the arrow indicates a KP-feasible knock, either directly on the boundary gadget or on the terminal gadget reached by following first blockers in the outward boundary directions.
}
    \vspace{-15pt}
    \label{fig:feasibility}
\end{figure}



\begin{cameraready}

\begin{definition}[$\sigma$-outer boundaries, outward directions, and boundary distance]
\label{def:sigma-outer-frontier}
Let $\G$ be a cleaned graph with nonempty face set $U=\Faces(\G)$, and let $\cC$ be an exact cover of $U$ by gadgets in $\cS$.
Fix $\sigma=(\sigma_1,\sigma_2)\in\{\pm1,\pm1\}$, and let $c_\sigma$ denote the corresponding hull corner.
If horizontal gadgets exist in $\cC$, choose
$
S_h^\sigma = R_h(i_h,j_h)
\in \arg\max \{(\sigma_1 i,\sigma_2 j): R_h(i,j)\in\cC\},
$
with lexicographic maximization. If vertical gadgets exist in $\cC$, choose
$
S_v^\sigma = R_v(i_v,j_v)
\in \arg\max \{(\sigma_2 j,\sigma_1 i): R_v(i,j)\in\cC\},
$
again with lexicographic maximization.
Define the outward boundary directions
$
d_h^\sigma\triangleq(\sigma_1,0),
d_v^\sigma\triangleq(0,\sigma_2).
$
These are the directions toward the hull corner $c_\sigma$ for horizontal and vertical gadgets, respectively.
If $S_h^\sigma=R_h(i_h,j_h)$ exists, define its \emph{horizontal $\sigma$-outer boundary coordinate} by
$
b_h^\sigma \;\triangleq\; i_h+\frac{1+\sigma_1}{2},
$
and if no horizontal gadget exists, let $b_h^\sigma$ be the coordinate of the hull edge incident to $c_\sigma$ in direction $d_h^\sigma$.
If $S_v^\sigma=R_v(i_v,j_v)$ exists, define its \emph{vertical $\sigma$-outer boundary coordinate} by
$
b_v^\sigma \;\triangleq\; j_v+\frac{1+\sigma_2}{2},
$
and if no vertical gadget exists, let $b_v^\sigma$ be the coordinate of the hull edge incident to $c_\sigma$ in direction $d_v^\sigma$.
The \emph{horizontal $\sigma$-outer boundary} and \emph{vertical $\sigma$-outer boundary} are
$
B_h^\sigma \;\triangleq\; \{(i,j)\in\cH : i=b_h^\sigma\},
B_v^\sigma \;\triangleq\; \{(i,j)\in\cH : j=b_v^\sigma\}.
$
For a gadget $S$, let $x_\sigma(S)$ denote its candidate vertex facing $c_\sigma$, namely
$
x_\sigma\!\left(Q(i,j)\right)
   =\left(i+\frac{1+\sigma_1}{2},\,j+\frac{1+\sigma_2}{2}\right),
$
$
x_\sigma\!\left(R_h(i,j)\right)
   =\left(i+\frac{1+\sigma_1}{2},\,j+1\right),
$
$
x_\sigma\!\left(R_v(i,j)\right)
   =\left(i+1,\,j+\frac{1+\sigma_2}{2}\right).
$
Define $\delta_\sigma(S)$ as the $\sigma$-distance of a gadget $S$, which is the sum of the distances from the $\sigma$-facing candidate of $S$ to the two hull edges incident to $c_\sigma$.
A \emph{first blocker} of a gadget in an outward direction $d\in\{d_h^\sigma,d_v^\sigma\}$ means any gadget whose vertices intersect the corresponding outward ray and whose first point of intersection occurs at minimum positive distance from the candidate vertex along that ray.

\end{definition}

\begin{lemma}[Existence of a KP-feasible gadget]
\label{lem:sigma-outer-feasible}
Let $\G$ be a cleaned graph with nonempty face set $U=\Faces(\G)$, and let $\cC$ be an exact cover of $U$ by gadgets in $\cS$.
Fix $\sigma\in\{\pm1,\pm1\}$, and let
$
S_h^\sigma,\,
S_v^\sigma,\,
B_h^\sigma,\,
B_v^\sigma,\,
d_h^\sigma,\,
d_v^\sigma,\,
x_\sigma,\,
\delta_\sigma
$
be as in Definition~\ref{def:sigma-outer-frontier}.
Then there exists a gadget $S^\star\in\cC$ and a candidate-direction pair
$
(v^\star,d^\star)\in \Cand(S^\star)\times\mathcal D
$
such that
$
\Ray(v^\star,d^\star)\cap V(\G)=\emptyset.
$
\end{lemma}

\begin{proof}
Fix $\sigma\in\{\pm1,\pm1\}$.
Let $H$ be a horizontal gadget on $B_h^\sigma$ maximizing $\sigma_2 j$ and let $V$ be a vertical gadget on $B_v^\sigma$ maximizing $\sigma_1 i$, whenever such gadgets exist.
If either $H$ or $V$ has an unobstructed outward boundary direction toward $c_\sigma$, then it is KP-feasible. Assume therefore that both $H$ and $V$, whenever they exist, are blocked in the outward directions $d_h^\sigma$ and $d_v^\sigma$. Starting from $H$ and $V$, successively take first blockers in the corresponding outward directions. At each step, the blocker does not increase the $\sigma$-distance $\delta_\sigma$ and moves strictly toward the hull corner $c_\sigma$ along at least one of the two directions $d_h^\sigma$ or $d_v^\sigma$. 
Intuitively, this holds since choosing any blocker towards $d_h^\sigma$ or $d_v^\sigma$ (which are orthogonal directions) is guaranteed to be a gadget whose $\sigma$-distance is not larger than the current gadget. Observe that these gadgets with non-increasing $\sigma$-distance are broadly contained within the diagonal region between the current gadget and the corner $c_\sigma$. With each iteration, this process will reduce the region or set of possible blockers to consider, since by structure backtracking is not possible.
Since the hull is bounded, this blocker process cannot continue indefinitely. Hence it terminates at a gadget $S^\star$ that has no blocker in one of its relevant outward directions, i.e., offering an unobstructed knock direction $d^\star$ for candidate $v^\star=x_\sigma(S^\star)$.
Therefore $S^\star$ is KP-feasible. 
Note that the guaranteed existence of such a gadget is of interest, not the precise process of finding this gadget or its type. Hence, showing that such a $S^\star$ should exist suffices for our arguments. 
\end{proof}

\end{cameraready}

\noindent




\begin{theorem}[Optimal and feasible knock set]\label{thm:knocks-equals-exactcover}
Let $\G^\circ=\Clean(\G(\cB))$ and $\cU=\Faces(\G^\circ)$.
Let $\cC^\star$ be a minimum-cardinality exact cover of $(\cU,\cS)$ obtained from a maximum-weight perfect matching of $\cG_F$. 
Then: (i) The minimum number of knocks in any complete KP execution equals $|\cC^\star|$; (ii) There exists a KP-feasible execution achieving exactly $|\cC^\star|$ knocks.
\end{theorem}

\begin{proof}

\emph{(i) Optimality: minimum knocks $=|\cC^\star|$.}
Consider any complete KP execution starting from $\G^\circ$.
For each face $f\in\cU$, let its first destroyed vertex be assigned to the knock that destroys it.
By Lemma~\ref{lem:face-needs-knock}, this first destruction must be a knock.
By Lemma~\ref{lem:knock-two-faces}, each knock destroys at most two faces, and two faces can be destroyed together only when they form a horizontal or vertical two-face gadget.
Hence the knocks induce a pairwise-disjoint family of sets from $\cS$ covering $\cU$, that is, an exact cover $\cC$ with $|\cC|\le \#\text{knocks}$. Therefore,
$
\#\text{knocks} \ge \min\{|\cC|:\cC\text{ exact cover}\} = |\cC^\star|.
$
Conversely, each set in an exact cover can be destroyed by one knock (Lemma~\ref{lem:gadget-oneknock}), so $|\cC^\star|$ knocks suffice in principle. This proves (i).

\emph{(ii) Every exact cover has a KP-feasible execution.}
Let $\cC$ be an exact cover of $\cU=\Faces(\G^\circ)$. We argue by induction on $|\cC|$.
If $\cU=\emptyset$, exhaustive picking removes all remaining vertices and the execution terminates.
Assume $\cU\neq\emptyset$, and fix any $\sigma\in\{\pm1\}^2$.
By Lemma~\ref{lem:sigma-outer-feasible}, the $\sigma$-outer boundary construction determines a KP-feasible gadget
$
S^\star\in\cC.
$
Execute knocking on $S^\star$ and then apply exhaustive picking.
By Lemma~\ref{lem:gadget-oneknock}, the face(s) covered by $S^\star$ are destroyed, so if
$
\cC' \triangleq \cC\setminus\{S^\star\},
$
then $\cC'$ is an exact cover of the remaining face set $\Faces(G')$, where $G'$ is the resulting cleaned graph.
By induction, $G'$ admits a complete KP execution using exactly $|\cC'|$ knocks.
Prepending the feasible knock of $(v^\star,d^\star)$ yields a complete KP execution for $G$ using exactly
$
1+|\cC'| = |\cC|
$
knocks.
Applying this to $\cC=\cC^\star$ yields a complete KP-feasible execution using exactly $|\cC^\star|$ knocks.

\end{proof}

\vspace{-20pt}
\subsection{Optimal KP Solver}
Theorem~\ref{thm:knocks-equals-exactcover} characterizes the optimal knock count via an optimal exact cover (equivalently, a maximum-weight perfect matching in $\cG_F$).
To produce an \emph{executable} KP sequence under the ray constraint~\eqref{eq:rayclear},
we use a local selection routine that checks ray feasibility and respects the gadget-dependent candidate sets.

\begin{algorithm}[t!]
\caption{\textsc{OPTKPSolver}}
\label{alg:execute}
\begin{algorithmic}[1]
\Require Set of blocks $\objects$ with their grid-arranged locations
\Ensure A complete KP execution $\pi$
\State $G \gets \textsc{getGridGraph}(\objects)$
\State $\pi_{\mathrm{PICKS}} \gets \textsc{pick}(G)$ \Comment{initial feasible pick sequence}
\State $\pi \gets \pi \oplus \pi_{\mathrm{PICKS}}$; \quad $G\gets G \setminus \pi_{\mathrm{PICKS}}$
\State $\cU \gets \textsc{getFaces}(G)$ 
\State $\cG_F \gets \textsc{buildFaceGraph}(\cU, G)$ 
\State $\cM^{\star} \gets \textsc{maxWeightPerfectMatching}(\cG_F)$
\State $\cC^{\star} \gets \textsc{edgeCoverToGadgets}(\cM^{\star})$ 
\While{$G.V \neq \emptyset$}
    \State $(v,d) \gets \textsc{chooseKnock}(\cC^{\star}, G)$ \Comment{next feasible knock}
    \State $\pi \gets \pi \oplus (v,d)$;\quad $G\gets G \setminus \{v\}$
    \State $\pi_{\mathrm{PICKS}} \gets \textsc{pick}(G)$ \Comment{next feasible pick sequence}
    \State $\pi \gets \pi \oplus \pi_{\mathrm{PICKS}}$; \quad $G\gets G \setminus \pi_{\mathrm{PICKS}}$
\EndWhile
\State \Return $\pi$
\end{algorithmic}
\end{algorithm}

Algorithm~\ref{alg:execute} alternates exhaustive picking with single feasible knocks. 
After initial cleanup, it extracts the face set, builds the face-adjacency graph, and computes an optimal perfect matching \(\cM^\star\), which yields a minimum-cardinality exact cover \(\cC^\star\) by Lemma~\ref{lem:exactcover-matching}. 
By Theorem~\ref{thm:knocks-equals-exactcover}, \(|\cC^\star|\) is the optimal knock count and \(\cC^\star\) admits a KP-feasible execution. 
The routine \(\textsc{chooseKnock}\) selects a ray-feasible candidate from the remaining gadgets of \(\cC^\star\), and each such knock is followed by exhaustive picking until the graph is empty. 
Hence the returned execution is complete and optimal in the number of knocks.

\section{Results}
\label{sec:benchmarks}

We evaluate \textsc{OPTKPSolver} in two settings: (i) a fast \emph{synthetic} benchmark that isolates the combinatorial structure, and (ii) an \emph{IsaacSim} benchmark that executes the resulting KP sequence with a Franka Panda arm under physics.
Across all experiments, the \emph{knock set} is computed by our \emph{optimal} reduction to maximum-weight perfect matching on the face graph (Section~\ref{sec:method}). Hence all reported knock counts are globally optimal for the cleaned instance.

\subsection{Synthetic Benchmark}
\paragraph{Setup.}
We consider two families of instances:
(i) \emph{full grids} of sizes $3\!\times\!3$, $5\!\times\!5$, $10\!\times\!5$, $10\!\times\!10$, $20\!\times\!10$, and $20\!\times\!20$; and
(ii) \emph{subgraphs} formed by uniformly deleting vertices from the corresponding full grid to obtain $|V|\in\{5,16,25,50,100,200\}$.
For each subgraph size we generate $N=20$ random instances and report averages.
All synthetic experiments are run on an AMD Ryzen 9 9950X 16-Core CPU.
The maximum-weight perfect matching is computed using the 
{NetworkX}~\cite{hagberg2007exploring} implementation of Edmonds' blossom algorithm~\cite{galil1986efficient}.

\begin{wraptable}[17]{r}{0.5\textwidth}
\centering
\vspace{-34pt}
\caption{Synthetic results (optimal knock count via maximum-weight perfect matching) on full grids and random subgraphs.}
\label{tab:synthetic}
\setlength{\tabcolsep}{4pt} 
\begin{tabular}{lrSS}
\toprule
grid dims & $|V|$ & {\#knocks} & {$t_{\text{total}}$ (ms)} \\
\midrule
 3x3   &   9 &  2.00  &   0.55 \\
 5x5   &  25 &  8.00  &   1.47 \\
 10x5  &  50 & 18.00  &   3.51 \\
 10x10 & 100 & 41.00  &  13.05 \\
 20x10 & 200 & 86.00  &  61.62 \\
 20x20 & 400 & 181.00 & 189.79 \\
\midrule
 3x3   &   5 &  1.00  &  0.13 \\
 5x5   &  16 &  3.12  &  0.47 \\
 10x5  &  25 &  7.00  &  1.43 \\
 10x10 &  50 & 14.04  &  3.03 \\
 20x10 & 100 & 27.36  &  9.40 \\
 20x20 & 200 & 60.16  & 36.02 \\
\bottomrule
\end{tabular}
\end{wraptable}
\paragraph{Metrics.}
We report (i) the optimal number of knocks, and (ii) total solver time
$t_{\text{total}}$ (ms), which includes computing the matching / exact cover and producing an executable knock--pick sequence.
The number of actions is always $|V|+\#\text{knocks}$ (each object is removed exactly once, and each knock adds one action), so we omit it from the tables.


\paragraph{Observations.}
Table~\ref{tab:synthetic} shows the statistics from the synthetic benchmark. On full grids, the optimal knock count closely tracks half the number of unit faces.
Indeed, a full $m\times n$ grid contains $|\cU|=(m-1)(n-1)$ faces, and the matching-based optimum prefers weight-$2$ edges whenever possible, yielding $\lceil |\cU|/2\rceil$ knocks (e.g., $20\times 20$ has $361$ faces and the optimum is $181$ knocks).
For subgraphs, random removal of vertices allows more blocks to be pickable and the face graph is sparser, reducing the knock-to-face ratio as expected.

In terms of computation, $t_{\text{total}}$ increases smoothly with instance size and remains well below a second up to $|V|=400$. This matches the polynomial-time nature of maximum-weight matching and the subsequent linear-time sequence construction in our pipeline.

\subsection{IsaacSim Benchmark}
\paragraph{Setup.}
We execute the KP sequences with a Franka Panda (7-DoF) and parallel gripper in Isaac Sim 5.0.0 using the PhysX engine. Motion planning is performed in ROS2~\cite{scirobotics.abm6074} and MoveIt2 task constructor~\cite{gorner2019moveit}, using RRTConnect~\cite{kuffner2000rrt} in OMPL~\cite{sucan2012open} for each pick/knock motion. Experiments are run on an RTX~5090 GPU workstation. We test full grids of size $3\times 3$ and $5\times 5$ (10 runs each), and random subgraphs with $|V|{=}5$ (from $3\times 3$) and $|V|{=}16$ (from $5\times 5$). For each subgraph size we use 5 random problems and repeat each 10 times (50 runs total).

\paragraph{Metrics.}
We report success rate, optimal knock count, computation time $t_{\text{comp}}$ (MoveIt/OMPL), execution time $t_{\text{exec}}$ (IsaacSim), all in seconds.
We additionally report \texttt{knRetry}, the mean number of knock retries per run: if a knock does not create sufficient clearance due to contact/friction/torque effects, the same knock motion is repeated until clearance is detected or the run is declared failed.

\begin{table}[t]
\centering
\caption{IsaacSim results. Knock counts are optimal for the cleaned instance. \texttt{knRetry} reports the mean number of physics-triggered knock retries per run.}
\label{tab:isaac}
\setlength{\tabcolsep}{3pt}
\begin{tabular}{lrSSSSSS}
\toprule
grid dims & $|V|$ & {run succ.} & {obj pick succ.} & {\#knocks} & {$t_{\text{comp}}$ (s)} & {$t_{\text{exec}}$ (s)} & {\texttt{knRetry}} \\
\midrule
3x3     & 9  & 1.0 & 1.00 & 2.0 &  4.50 &  92.07 & 0.70 \\
5x5     & 25 &  0.8 & 0.97 & 8.0 & 13.56 & 272.48 & 5.00 \\
\midrule
3x3      & 5  & 1.0 & 1.00 & 1.0 &  2.44 &  51.03 & 0.56 \\
5x5      & 16 & 1.0 & 1.00 & 3.6 &  8.49 & 165.41 & 1.44 \\
\bottomrule
\end{tabular}
\end{table}

\paragraph{Observations.}
Table~\ref{tab:isaac} shows the statistics of the IsaacSim benchmark. Execution dominates total time: physical simulation execution time $t_{\text{exec}}$ is one to two orders of magnitude larger than motion planning time $t_{\text{mp}}$ across all cases, since each action requires approach, contact, and retraction under physics.
While the \emph{number of actions} scales as $|V|+\#\text{knocks}$, the wall-clock execution time is also sensitive to physical interactions during knocks.

The \texttt{knRetry} column makes this explicit: larger and denser instances (notably full $5\times 5$) require more retries on average, reflecting the fact that a nominally valid knock (ray-feasible in the discrete model) may under-deliver displacement when realized with a finite-stiffness gripper and frictional contacts. 
This same effect explains why the full $5\times 5$ case exhibits runs where primitive failures were recorded: some runs enter unrecoverable contact configurations where repeated knock attempts cannot produce sufficient clearance. Despite primitives failing in two runs (run succ. of $0.8$), proceeding with the remaining actions succeeded to exhaustively pick almost all of the objects (pick success of $0.97$). 


\section{Discussion}

This work studies grid-aligned, tightly packed tabletop rearrangement where purely prehensile removal with a parallel gripper can be blocked. Allowing a ray-constrained directional knock enables progress, and the minimum-knock objective admits an exact combinatorial characterization: after cleanup, optimal knock sets correspond to minimum exact face covers, computed optimally via matching on the face graph. Executability is recovered by interleaving one ray-feasible knock (chosen from gadget-specific candidate vertices) with exhaustive picking, which monotonically reduces the instance until all blocks are removed.

\paragraph{Empirical observations.}
The synthetic benchmark indicates that the optimal solver scales smoothly with instance size and that runtime grows polynomially with the number of objects/faces, consistent with the matching-based reduction. In Isaac Sim, the same optimal knock counts are realized, while total wall-clock time is dominated by physics-based execution rather than the combinatorial solver. Knock retries were occasionally required: even when ray-clearance holds, contact dynamics (e.g., friction, gripper compliance, impulse transfer) can prevent a block from fully leaving its pocket. This highlights the gap between discrete clearance abstractions and continuous rigid-body interactions.

\paragraph{Limitations.}

The setting being studied is heavily simplified --- top-down reachable, uniformly sized, blocks arranged in a grid, with hand-crafted non-prehensile primitives. The setting allows us to glean theoretical insights but extending to practical robust performance still poses open, challenging questions. The Sim-to-Real gap will be a persistent bottleneck in applying combinatorial graphical abstractions to real-world executions with physics. The primitive itself is simplistic and represents the minimal operation needed to locally relax adjacency constraints. There may be more sophisticated contact-rich non-prehensile primitives, including aggregating push strategies that may benefit some tasks. Bespoke fingered end-effectors may pose their own constrained abstraction and complexity. The placement side of the problem is also assumed to be unconstrained, which though valid in non-overlapping situations, need to be explored in general overlapping scenaros. 
Moreoever, theoretical insights into generalized object shapes start converging on typical considerations within general assembly planning and it is to be seen whether the current theoretical efficiencies can be used to speed up more general settings. 

Despite the current result being a preliminary step, the theoretical insight into the studied category of problems permitting computationally efficient solutions to compute optimal task plans, promises opportunities to build principled theoretically-driven practically-robust strategies that can seamlessly switch between diverse prehensile and non-prehensile strategies that balance practical performance with theoretical guarantees. 

\bibliographystyle{splncs04}
\bibliography{mybibliography}

\end{document}